\def\year{2016}\relax
\documentclass[letterpaper]{article}
\usepackage{aaai16}
\usepackage{times}
\usepackage{helvet}
\usepackage{courier}
\usepackage{graphicx}
\usepackage{epstopdf}
\usepackage{multirow}
\usepackage{tabularx}
\usepackage{url,subfigure,amsmath,amssymb,epsfig,verbatim,scrextend}
\usepackage[noend,linesnumbered,ruled,vlined,boxed]{algorithm2e}
\usepackage{amsthm}
\providecommand{\ourparagraph}[1]{\vspace{0.2em}\noindent\textbf{#1. }}

\providecommand{\pmaxq}{Parallel Max-Q }
\providecommand{\se}{SingleExpert }

\providecommand{\bse}{Best SingleExpert}

\theoremstyle{plain}
\newtheorem{thm}{\protect\theoremname}
\theoremstyle{plain}
\newtheorem{lem}{\protect\lemmaname}
\ifx\proof\undefined
\newenvironment{proof}[1][\protect\proofname]{\par
\normalfont\topsep6\p@\@plus6\p@\relax
\trivlist
\itemindent\parindent
\item[\hskip\labelsep\scshape #1]\ignorespaces
}{%
\endtrivlist\@endpefalse
}
\providecommand{\proofname}{Proof}
\fi
\theoremstyle{definition}
\newtheorem{defn}{\protect\definitionname}
\theoremstyle{plain}
\newtheorem{cor}{\protect\corollaryname}

\providecommand{\corollaryname}{Corollary}
\providecommand{\definitionname}{Definition}
\providecommand{\lemmaname}{Lemma}
\providecommand{\theoremname}{Theorem}

\nocopyright
\providecommand{\citet}[1]{\citeauthor{#1}~(\citeyear{#1})}
\theoremstyle{plain}
\newtheorem{theorem}{Theorem}
\begin{document}

\title{Scaling POMDPs For Selecting Sellers in E-markets---Extended Version}
\author{Athirai A. Irissappane\\
Nanyang Technological University\\
 Singapore\\
athirai001@e.ntu.edu.sg\\
\And
Frans A.\ Oliehoek\\
University of Amsterdam\\
 University of Liverpool\\
frans.oliehoek@liverpool.ac.uk\\
\And
Jie Zhang\\
Nanyang Technological University\\
 Singapore\\
zhangj@ntu.edu.sg\\
}
\maketitle

\begin{abstract}
In multiagent e-marketplaces, buying agents need to select good sellers by querying other buyers (called advisors).
Partially Observable Markov Decision Processes (POMDPs) have shown to be an effective framework for optimally selecting sellers by selectively querying advisors.
However, 
current solution methods do not scale to hundreds or even tens of agents operating in the e-market. In this paper, we propose the Mixture of POMDP Experts (MOPE) technique, which exploits the inherent structure of trust-based domains, such as the seller selection problem in e-markets, by aggregating the solutions of smaller sub-POMDPs. 
We propose a number of variants of the MOPE approach that we analyze theoretically and empirically.
Experiments show that MOPE can scale up to a hundred agents thereby leveraging the presence of more advisors to significantly improve buyer satisfaction. 
\end{abstract}

\section{Introduction}\label{introduction}

In many domains, agents need to determine the trustworthiness (quality) of other agents before interacting with them. Specifically, in e-marketplaces, buying agents need to reason about the quality of sellers and determine which sellers to do business with (referred to as the \emph{seller selection problem}). When buyers have no previous experience with sellers, they can obtain advice by querying other buyers (called \emph{advisors}). However, some advisors may be untrustworthy and provide misleading opinions to promote or demote the sellers~\cite{irissappane2015filtering}.

The Partially Observable Markov Decision Process (POMDP) is a framework for sequential decision making under uncertainty, suitable for e-markets, where buyers often need to make decisions with limited information about the sellers and advisors. \citet{regan2005advisor} propose the Advisor POMDP, for the seller selection problem, which, rather than trying to achieve the most accurate estimate of sellers, tries to select good sellers optimally with respect to its belief. Seller and Advisor Selection (SALE) POMDP~\cite{irissappane2014pomdp} extends Advisor POMDP to additionally deal with trust propagation, by introducing queries about advisors.
The SALE POMDP formalism
enables maximizing buyer satisfaction by optimally trading off information
gaining (querying advisors) and exploiting (selecting a seller) actions, and
experiments have shown very good results in practice. Also, the approach is
easily generalizable to deal with more general problems with trust-propagation
components, such as routing in Wireless Sensor
Networks (WSNs)~\cite{irissappanesecure}.

Unfortunately, these POMDP approaches suffer from scalability issues
. Finding
optimal policies for POMDPs is, in general, computationally intractable (PSPACE
complete) and POMDP solvers computing exact solutions, e.g., value iteration do not scale to more than a handful of states~\cite{cassandra1994acting}. 
While approximation algorithms have been shown to supply good policies rapidly 
even for problems with very large state spaces~\cite{spanpomdp,silver2010monte},
the scalability of the SALE POMDP, which is based on one such method~\cite{poupart2005exploiting}, is limited to about $10$ agents
(sellers and advisors). For larger number of agents, the solution time grows to the
order of hours and solution quality degenerates, precluding the SALE POMDP from
exploiting the presence of more sellers and advisors.


This paper proposes 
a novel method, referred to as the \emph{Mixture of POMDP Experts (MOPE)}
approach, for dealing with very large trust-propagation problems such as SALE
POMDPs with many sellers and advisors.
The key idea is to divide the large seller selection POMDP problem into a
multitude of computationally tractable smaller \emph{(sub)-POMDPs}, each
containing a subset of sellers and advisors. The actions of the
\emph{sub-POMDPs} (SPs) are then aggregated, to find the best action in the
process of selecting a good seller.

The MOPE approach exploits the structure of the Dynamic Bayesian Network 
that represents the transition and observation probabilities of the SALE POMDP:
query actions do not affect the actual states but only the agent's
\emph{beliefs} over the state factors, 
making it easier to decompose a large seller selection problem into smaller sub-problems that approximate the
larger problem.
Due to the improved scalability of MOPE, it can leverage the
presence of more advisors to make more informed decisions about sellers,
when the size of the seller selection problem increases.
Extensive evaluation in
a simulated e-marketplace demonstrates that MOPE can scale up to a hundred
agents (millions of states and thousands of actions), outperforming
the state-of-the-art POMCP~\cite{silver2010monte} approach, while using less computation time.
%
We also demonstrate that MOPE can bring scalability to other domains
by showing results for wireless sensor networks with up to $40$ neighboring nodes.

\section{Background}\label{sec:single-pomdp}

This paper mainly relies on POMDPs,
which can be used to represent decision making problems under uncertainty in terms of
states, actions, transitions, observations and rewards. We refer to Kaelbling et al.~\shortcite{kaelbling1998planning}, Spaan~\shortcite{spanpomdp} for a comprehensive
introduction to POMDPs. Here, we try to convey the most basic intuitions by
briefly describing the Seller and Advisor Selection (SALE) POMDP~\cite{irissappane2014pomdp}, which is the main application for the technique
we propose in this paper.

\ourparagraph{States}Each state is represented using a number of state factors, such as the quality levels of each seller ($q_j\in\{high, low\}$), each advisor ($u_i\in\{trustworthy, untrustworthy\}$) and status of the transaction ($sat\in\{not\_started$, $satisfactory$, $unsatisfactory$, $gave\_up$, $finished\}$).

\ourparagraph{Actions and Transitions}For query actions such as $seller\_query_{(i,j)}$ ($(SQ)_{(i,j)}$), i.e., ask advisor $i$ about seller $j$ and $advisor\_query_{(i,i')}$ ($(AQ)_{(i,i')}$), i.e., ask advisor $i$ about another advisor $i'$, the states do not change. For $BUY_{j}$ action, to buy from seller $j$, the state transitions to successful ($sat$ = $satisfactory$) on buying from a good seller and unsuccessful ($sat$ = $unsatisfactory$) on buying from a bad seller. For $do\_not\_buy$ ($DNB$) action, i.e., do not buy from any seller, the state transitions to $sat$ = $gave\_up$.

\ourparagraph{Rewards}
There is small cost for the query actions. A reward/penalty is associated with a successful/unsuccessful transaction. There is a penalty for taking $DNB$ action, when there is a seller of high quality, otherwise a reward is given.

\ourparagraph{Observations} After $SQ_{ij}$, $AQ_{ii'}$ actions, an observation $o\in \{good, bad\}$, corresponding to the quality of seller~$j$ and $o\in \{trustworthy, untrustworthy\}$ corresponding to the quality of advisor~$i'$ is received, respectively. After $BUY_j$ action, the agent can also receive an observation based on the actual quality of seller~$j$, allowing to reuse the updated beliefs, in case of multiple transactions. The observation probabilities are such that trustworthy advisors give more accurate and consistent answers than untrustworthy ones.

\providecommand{\belS}{B}

The SALE POMDP agent interacts with its environment for an indefinite number of
time steps and we model the problem using an infinite horizon.
During this interaction, the agent maintains a \emph{belief}
$b \in \belS$, i.e., a probability distribution over states. If $b(s)$ specifies the probability of $s$ (for all $s$), we can derive $b'$ an updated belief after taking some
action $a$ and receiving an observation $o$ using the Bayes' rule.
%
A POMDP policy $\pi: \belS\rightarrow\mathcal{A}$, maps
belief $b \in \belS$ to an action $a \in \mathcal{A}$. A
policy $\pi$ is associated with a value function $V(b)$,
specifying the expected total reward of executing policy $\pi$
starting from $b$, with discount factor $\gamma$. The main objective of the POMDP agent is to find an optimal policy $\pi^*$, which maximizes $V(b)$ (Eqn.~\ref{eq:valuefun}). The value function can also be represented in terms of Q-functions, given by Eqn.~\ref{eq:qfun}, where, $b^a_o$ is the belief state resulting from $b$ after taking action $a$ and receiving observation $o \in \mathcal{O}$.
\begin{equation}
\hspace{-0.5mm}V^*(b)\hspace{-1mm}=\hspace{-1mm}\max_{\pi} \mathbb{E}\Big[\sum_{t}\gamma^t R(s,a,s')\hspace{-1mm} \mid \hspace{-1mm}\pi,b\Big]\hspace{-1mm}=\hspace{-1mm}\max_{a\in A} Q^*(b,a)
\label{eq:valuefun}
\end{equation}
\begin{equation}
Q^*(b,a)=\sum_{s\in S}b(s)R(s,a) + \gamma \sum_{o \in \Omega} p(o|b,a) V^*(b^a_o)
\label{eq:qfun}
\end{equation}

By computing the optimal value function, we can optimize \emph{long-term} rewards by picking maximizing actions. This stands in contrast to \emph{myopic} approaches that maximize the immediate rewards $R$. Such approaches are inherently unsuitable for seller selection: in order to correctly value the different query actions, one needs to reason about their impact on the future beliefs and the associated value of information. In order to actually compute $V^*$ (approximately) one could rely on state-of-the-art flat solvers such as SARSOP~\cite{Kurniawati08RSS}, but these do provide very limited scalability~\cite{Oliehoek12TRUST}, since the number of states grows exponentially with the number of agents $n$ (i.e., sellers and advisors). Therefore, \citet{irissappane2014pomdp} employ a solution method,  factored Perseus~\cite{poupart2005exploiting}, that exploits the factored representation of this domain, thus allowing to scale to roughly 10 agents. Beyond that solution times go up significantly while solution quality drops. Apart from the number of state factors themselves, a difficulty is that the number of actions grows with order $O(n^2)$ as the query actions involve pairs of agents.

\section{A SingleExpert Baseline}
\label{sec:se}


In this paper, we propose techniques to exploit the structure
present in (settings like) the SALE POMDP. Here, we introduce a baseline algorithm as an intuitive starting point for the more advanced method we propose in the next section.

This baseline, called SingleExpert, is basically a method to apply the SALE POMDP for large problems. That is, when faced with a SALE POMDP instance with many
sellers and advisors, we can randomly select a subset of agents that is small enough to model and
solve as a SALE POMDP and use the resulting policy. Since the $q_j$ and $u_i$ variables do not
influence each other, defining such a \emph{sub-POMDP (SP)} is trivial as it merely amounts to
deleting all non-selected state variables as well as actions and observations that pertain to them.
Also, the resulting model is a small SALE POMDP, thus we can find a good solution for it.
While this voluntary restriction on the set of sellers and advisors that will be
reasoned about could be limiting,
it is quite possible that it may lead to acceptable performance and it might be better
than incorrectly reasoning about all of the agents.
We call this approach the `SingleExpert' approach, since the randomly selected
SP acts as a (single) expert as to what action to take in the larger problem.

\section{Mixture of POMDP Experts (MOPE)}\label{sec:moe}
While we argue that SingleExpert might have its merit, clearly, we want to develop methods that can
exploit large pools of potential sellers and advisors. To accomplish this, we introduce the
\emph{Mixture of POMDP Experts (MOPE)} framework.
SingleExpert exploits a particular property of trust propagation-like domains: constructing
an SP is possible because the state variables encoding seller and advisor qualities do not
affect each other and cannot be influenced by actions.
In fact, interaction of these variables only arises in the agent's beliefs manifested as
correlations induced by the coupling via observations. For example, if we query advisor~$i$ about
seller~$j$ and receive observation $bad$, it not only increases the probability of the seller being
low quality ($q_j=low$) and advisor being $u_i=trustworthy$, but also of
$(q_j=high,u_i=untrustworthy)$. The MOPE framework aims to take this insight further by approximating such
correlations using smaller clusters of variables, as in variational inference
approaches~\cite{koller2009probabilistic}, leading to the idea of representing the larger problem
using a number of smaller SPs and leveraging their solutions.
That is, rather than considering a single expert, we will want to consider
many SPs.


\subsection{MOPE Algorithm Overview}\label{subsec:overview}
Algorithm~\ref{alg:pomdpalg1} gives a brief overview of the MOPE framework. We first form the SPs by randomly selecting a subset of sellers and advisors ($M_k$ in Line $1$). Each SP is solved to obtain the optimal policy and thereby its maximum expected total reward $V_k^*$ (Line $2$)\footnote{
In practice, we may not solve the SPs optimally, and use the best policy and accompanying value function that we could find.}.
When SPs have the same agent composition (number of sellers and advisors), the
found $V_k^*$ can be reused amongst them.
Therefore, in our implementation we always select such uniformly composed
SPs.
We define $\mathcal{V}$ as a set of votes $v$ collected from each SP. Each vote $v=(a,q)$ is a set containing the action $a$ suggested by the SP and its associated Q-value $q$.

To maintain beliefs about all the state factors, it is possible to maintain the local beliefs in each SP, in parallel. However, doing so: 1) we need to deal with the actions not present in a SP as its local belief will be updated only if the SP contains the executed action $\bar{a}$; 2) we cannot properly take into account the influence of state factors not modeled in the SP on the belief, which may lead to inconsistent beliefs in different SPs. Thus, we propose to maintain and update the beliefs $b\in \belS$ at the global level, i.e., involving all state factors.

At each time step, for each SP, we first extract its current local belief $b_k$ (Line $6$) from the global belief $b$. Based on $b_k$, we obtain its vote $v$, i.e, its recommended action $a$ and the associated $q$ value (Line $7$). The overall best action $\bar{a}$ is obtained (Line $8$) by aggregating all the votes $v \in \mathcal{V}$. Action $\bar{a}$ is then executed (Line $9$) and an observation $o$ is received (Line $10$), based on which the global beliefs are updated (Line $11$). The following subsections give a more detailed description of the techniques used in the framework.

\begin{algorithm}[tb]
\label{alg:pomdpalg1}
\SetKwInOut{Input}{Input}
\SetKwInOut{Output}{Output}
\SetKwInOut{Define}{Define}
\SetKw{IF}{if}
\SetKwComment{Comment}{//}{}
\Input{
    $\mathcal{M}$, a large SALE POMDP
}
\BlankLine
Randomly split $\mathcal{M}$ into SPs $\{ \mathcal{M}_1, \dots, \mathcal{M}_K \}$ \\
Solve all SPs, yielding  $\{ V_1^*, \dots, V_K^* \}$\\
\ForEach{TimeStep $t$}{
    $\mathcal{V} \leftarrow \emptyset$ \Comment*{the set of votes}
    \For{$k \in \{1 \dots K\} $}{
        $b_k \leftarrow \text{DetermineLocalBelief}(b,k) $\\
        $\mathcal{V} \leftarrow \mathcal{V} \cup \{ \text{VoteFromSP}(k,b_k,V_k^*) \}$
    }
    $ \bar{a} \leftarrow \text{AggregateVotes}(\mathcal{V}) $\\
    Execute($\bar{a}$) \\
    $ o \leftarrow \text{receiveObservation}() $ \\
    $ b' \leftarrow \text{GlobalBeliefUpdate}(b,\bar{a},o) $ \\
}
\caption{The Mixture of POMDP Experts (MOPE) framework.}
\label{alg:pomdpalg1}
\end{algorithm}

 \subsection{Dividing into Sub-POMDPs}\label{sec:divide}
We randomly select subsets of sellers and advisors from the whole population $W$ to decompose $\mathcal{M}$ into a number of SPs. If SPA is the number of SPs that each agent can be a part of and APS is the number of agents each SP should contain, the total number of SPs necessary for the seller selection problem is given by $|W|$*SPA/APS. Also, APS is chosen such that the SPs can be computationally tractable.

\subsection{AggregateVotes($\mathcal{V}$)} Here, we describe different ways to aggregate the votes $\mathcal{V}$.

\ourparagraph{Parallel Max-Q} Here, the best action $\bar{a}$ is selected as the
action with the maximum Q-value ($\bar{a}=\arg \max_{a \in \mathcal{V}} q$),
among those present in $\mathcal{V}$. Also, Parallel Max-Q maintains, in
parallel, a set $\mathcal{B}=\left\{ b_{1},\dots,b_{K}\right\}$ of local beliefs
(corresponding to the SPs). We will use $\mathcal{B}$ as the global belief, in
this case. GlobalBeliefUpdate($b$,$\bar{a}$,$o$) is performed such that the
beliefs $b_{k}$ in each SP are updated using the Bayes' rule in parallel, when
$\ensuremath{\bar{a}\in\mathcal{A}_{k}$ (actions in $M_k$) and
$\ensuremath{o\in\mathcal{O}_{k}}}$. When either of these conditions fails,
no belief update takes place.


{
To analyse the performance of Parallel Max-Q for a given decomposition
$D=\left\{ \mathcal{M}_{1},\dots,\mathcal{M}_{K}\right\}$ of SPs, we derive a
lower bound on its performance. Specifically, we show that the expected sum of
rewards $V_{D}^{pmq}$ realized by parallel Max-Q for a decomposition $D$, is at
least as much as the optimal value $V_{k}^{*}$ realized by picking any 
SingleExpert $\mathcal{M}_{k}\in D$. For this, we need to make two assumptions:
the decomposition $D=\left\{\mathcal{M}_{1},\dots,\mathcal{M}_{K}\right\} $ is non-overlapping
(i.e., no two SPs $\mathcal{M}_{i},\mathcal{M}_{j}$ contain the same seller or advisor
state factors),
and
the true initial state distribution $\beta^{0}(s)$ is
factored along the decomposition
(i.e., $\beta^{0}(s)=\beta_{1}^{0}(s_{1}) \times \beta_{2}^{0}(s_{2}) \cdot\dots\times \beta_{K}^{0}(s_{K})$).

\begin{theorem}
    \label{th:val}
If the decomposition $D=\left\{ \mathcal{M}_{1},\dots,\mathcal{M}_{K}\right\} $
is non-overlapping, and the true initial state distribution $\beta^{0}$
is factored along the decomposition, then the value realized by \pmaxq{}
is at least as much as the value of the best Single Expert:
$
V_{D}^{pmq}(\mathcal{B}^0)
\geq
\max_{k\in\left\{ 1,\dots,K\right\} }V_{k}^{*}(\beta_{k}^0).
$
\end{theorem}

The proof of Theorem~\ref{th:val} is given in Appendix A, which includes a detailed theoretical analysis on the value realized by Parallel Max-Q. However, when the SPs are overlapping, Parallel Max-Q at times, can perform worse than the Best SingleExpert due to inconsistent beliefs across SPs. Imagine that there is a decomposition with two SPs (SP $1$, SP $2$) with a high overlap:
there is one seller which is present in both SPs, but each SP has
some private advisors. Also, assume that there is an uniform initial
belief such that the values of the SPs are equal, and \pmaxq{} selects
a `winning' SP, say SP $1$, randomly. Subsequently, the
executed action is a seller query that asks one of the private advisors
of SP $1$ about the (shared) seller, and if the answer is `bad', the belief of SP $1$ gets updated to reflect a lower probability
of the seller being high quality. Next, however, \pmaxq{} will switch
to SP 2 where the belief has not been altered. As such SP 2 is overestimating
the value because its belief is no longer is in sync with the
true distribution. This overestimation of
the value may lead to unnecessary information gaining actions which
have costs associated with them and eventually may lead to \pmaxq{} performing
worse than the \bse.


\begin{figure}[t]
\centering
 {\epsfig{figure=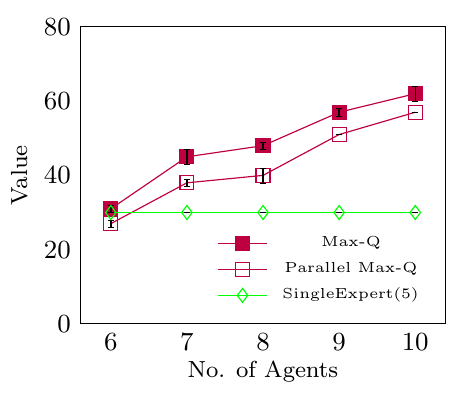, scale=1.4}}
\caption{Max-Q vs Parallel Max-Q}
\label{fig:votingH1}
\end{figure}

\ourparagraph{Max-Q} To address the issue of inconsistent beliefs, here in Max-Q, the beliefs are not maintained in parallel, instead they are maintained and updated (using Bayes' rule) at the global level, i.e., involving all state factors, as it helps to propagate information (about the sellers and advisors) across SPs. We empirically show the advantage\footnote{\label{foot1} Results are statistically verified by paired t-test ($\alpha$=$0.05$).} of Max-Q over Parallel Max-Q, for a decomposition $D$ in Fig.~\ref{fig:votingH1} by plotting their values for different seller selection problems, comprising of $6-10$ agents. We also show the value of SingleExpert (randomly chosen $5$ agent SP) in Fig.~\ref{fig:votingH1}.

\ourparagraph{Majority Voting} As Parallel Max-Q and Max-Q select the action of the maximizing SP (with the maximum Q-value), they consider the value that the action will generate for a single sub-problem. It is likely that certain actions are more useful for many sub-problems and it is better to select the action with a higher value in all SPs than the action with the highest value in one single SP. Here, we formalize one such technique called the Majority Voting approach.

Algorithm~\ref{alg:pomdpalg2} describes the Majority Voting approach in detail. Based on the votes $v \in \mathcal{V}$, we first count the number of SPs which suggested the action $a$ (Line $3$). We also determine the mean Q-value associated with each action $a$ using the qvalsum[] and meanQs[] variables (Lines $4-6$).

While using the Majority Voting technique, we need to consider the fact that not every SP will have the same set of actions as it depends on which sellers and advisors are present in the SP. For instance, while each SP has the action $DNB$, the action, say $SQ_{(a12,s23)}$ will only be present in a SP containing both $advisor12$ and $seller23$. Thus, most $SQ_{i,j}$ and $AQ_{i,i'}$ actions might not be represented in any SP, and the ones present may be represented in just one SP.


To address this, we make use of the additional information present in the actions of each SP by formulating the concept of \textit{abstract} actions. Consider the case where the belief indicates that there is a reasonable chance that $seller23$ is of high quality, but it falls just short of being sufficient to select the $BUY_{23}$ action. Here, it is very likely that all query actions that ask about $seller23$ that are represented in some SPs (we will denote this set by $SQ_{(X,s23)}$, where `X' denotes an unbound variable) will have a high value in those SPs. As such, voting on abstract actions (called Level L1 abstract actions), such as $SQ_{(X,s23)}$, $SQ_{(a12,Y)}$, $AQ_{(X,a30)}$, $AQ_{(a12,Y)}$, $BUY_{(Y)}$, $DNB$ can potentially help overcome the problem of sparsely represented actions.

However, only SPs which contain $seller23$ will have a $SQ_{(X,s23)}$ action, still resulting in unbalanced voting. Thus, rather than only abstracting away just one argument, we abstract away both arguments leading to abstract actions $SQ_{(X,Y)}$ and $AQ_{(X,Y)}$. Doing this leads to a situation where every SP has abstract actions $SQ_{(X,Y)}$, $AQ_{(X,Y)}$, $BUY_{(Y)}$ and $DNB$ (called Level L2 actions). But, still there may arise scenarios where $DNB$ actions can outnumber the $SQ_{(X,Y)}$, $AQ_{(X,Y)}$, $BUY_{(Y)}$ actions, individually, especially in cases when all good sellers are concentrated only to a group of SPs. Thus we consider (Level L3) abstract actions $DNB$ and $Others \in \{SQ_{(X,Y)}, AQ_{(X,Y)}, BUY_{(Y)}\}$, resulting in a $3$ level abstraction hierarchy shown in Fig.~\ref{fig:votingH3}.

\IncMargin{0.5em}
\begin{algorithm}[t]
\label{alg:pomdpalg3}
\SetKwInOut{Input}{Input}
\SetKwInOut{Output}{Output}
\SetKwInOut{Define}{Define}
\SetKw{IF}{if}
\SetKwComment{Comment}{//}{}
\SetKw{Return}{return}
\Input{$\mathcal{V}$, the set of votes}
\BlankLine
\Comment{Count votes for regular actions}
\ForEach{$v \in \mathcal{V}$}{
    $(a,q) \leftarrow v$ \Comment*{unpack vote}
    counts[$a$] += 1\; qvalsum[$a$] += $q$\;
}
\ForEach{$a \in \mathcal{A}$}{
    meanQs[$a$] = qvalsum[$a$] / counts[$a$]\;
}
    \Comment{Count votes for abstract actions}
    \ForEach{$v \in \mathcal{V}$}{
        $\mathcal{AV} = \{(\tilde{a},q)\} \leftarrow \text{AbstractedVotes}(v)$\;
        \ForEach{ $(\tilde{a},q) \in \mathcal{AV}$}{
            counts[$\tilde{a}$] += 1\; qvalsum[$\tilde{a}$] += $q$\;
        }
    }

    \ForEach{$  \tilde{a} \in \tilde{\mathcal{A}}  $}{
        meanQs[$\tilde{a}$] = qvalsum[$\tilde{a}$] / counts[$\tilde{a}$]\;
    }

    \Comment{Select best abstract action and refine}

    $ \tilde{a}^* = \arg \max_{\tilde{a}} ( \text{counts}[\tilde{a}] * \text{meanQs}[\tilde{a}]$ )\;
    \label{alg:best_abstract_action}
    $ \bar{a} = \text{Refine}( \tilde{a}^*, \text{counts}, \text{meanQs} )$\;
\Return{$\bar{a}$}
\caption{AggregateVotes by Majority Voting}
\label{alg:pomdpalg2}
\end{algorithm}

The L1, L2 and L3 abstract actions lead to new questions about
which ones should be included in the Majority Voting technique. In this work, we
empirically investigate these questions by considering a number of so-called voting hierarchies. In H1 hierarchy, only L1 abstract actions are considered and the best abstract action $\tilde{a}^*$ is chosen, after which the concrete action $\bar{a}$ is chosen. In hierarchy H2, first the best abstract action among the L2 abstract actions is determined, followed by the best L1 abstract action and finally the concrete action. In hierarchy H3 (shown in Fig.~\ref{fig:votingH3}), first the best L3 abstract action is determined followed by L2, L1 best abstract actions and then the concrete action.

\begin{figure}[t]
\centering
   \subfigure[H3]              {\includegraphics[height=3.98cm, width=6.32cm]{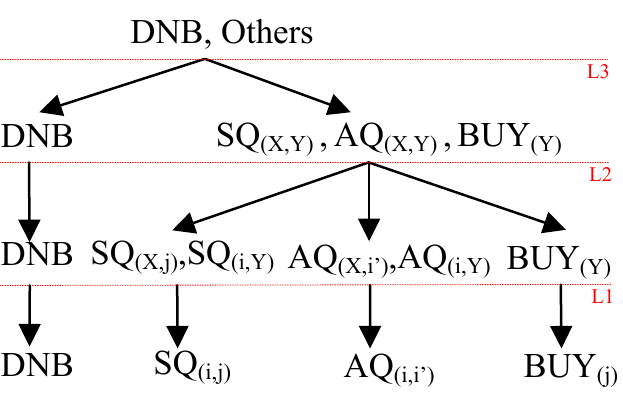}}
\caption{H3 Voting hierarchy}
\label{fig:votingH3}
\end{figure}

In Algorithm~\ref{alg:pomdpalg2}, we maintain a separate set of votes
$\mathcal{AV}$ for abstract actions $\tilde{a}$. In Line $8$, we determine
all the abstract actions that correspond to the regular action $a$ contained in vote $v$.
Subsequently, we increment their counts[$\tilde{a}$] and meanQs[$\tilde{a}$] (Lines $9$-$13$).
Then, in Line~\ref{alg:best_abstract_action}, the  best abstract action
$\tilde{a}^*$ is first selected, which
is subsequently refined to determine the best concrete action.
This refinement process depends on the employed voting hierarchy. For instance
when using the H1 hierarchy,
\[
    \bar{a} = \arg \max_{a \in \mathcal{A}(\tilde{a}^*)} \text{counts}[a] * \text{meanQs}[a],
\]
where $\mathcal{A}(\tilde{a}^*) $ denotes the set of concrete actions consistent
with abstract action $\tilde{a}^*$.


\subsection{Belief Update}\label{subsec:bel}

Though we can maintain and perform exact belief updates at the global level, i.e., involving all state factors, using the Bayes' rule, such exact inference is complex and does not scale to more than $10$ agents.
Therefore, we propose to employ the approximate inference methods.
In particular, we apply Factored Frontier (FF)~\cite{murphy2001factored}, which
maintains the belief in fully factored form, i.e., as the product of marginals of state factors $x_i$:
$b(s)= \prod_{i=1}^{|s|} \hat{b}(x_i)$. Thus the beliefs for each SP can directly be extracted via $b_k(s)=\prod_{x_i \in X_k} \hat{b}(x_i)$, where $ X_k$ denotes the set of state factors that are a part of the sub-POMDP $M_k$. While FF is a simple algorithm, and other choices are possible, it does allow influence of variables to propagate through the network and our experiments suggest that FF performs quite well.


\section{Experiments}\label{sec:exp}
%

Here, we empirically investigate the solution quality and scalability of the proposed
Mixture of POMDP experts (MOPE) technique in the e-marketplace domain. 
We are primarily interested to see if the added scalability can actually
translate into additional value from the buyer's perspective.

\ourparagraph{Experimental Setup}
We analyze different design considerations for MOPE (SPA=$4$ SPs per agent and APS=$5$ agents per SP with a uniform composition for all SPs comprised of $1$ seller and $4$ advisors, such that we can reuse V*, as described in Sec.~\ref{subsec:overview}) and compare it with:
\begin{enumerate}
\item the original SALE POMDP. We assume uniform initial beliefs and compute the SALE POMDP optimal policy using Symbolic Perseus~\cite{poupart2005exploiting}.
 \item SingleExpert(5), i.e., a randomly selected $5$-agent SP, serving as the lower bound.
  \item POMCP~\cite{silver2010monte}, an online planning approach which requires a number of random simulations (we use $10,000$ simulations per selected action) to estimate the potential for long-term reward.
       \item an \emph{optimistic heuristic} value $V_{maxv}$, which is the value obtained by running many simulations of MOPE (Majority Voting with H3 hierarchy and SPA=$8$) on `ideal' global problems (i.e., on $100$ agent problems with good sellers and trustworthy advisors). We consider such a heuristic as we know that beginning with a most favourable state (which in our case is the presence of good sellers and trustworthy advisors in the SPs, as they have a higher probability of resulting in successful transactions), results in best performance while executing a POMDP policy.
            \item the Q-MDP value $V_{qmdp}$, which is the value obtained by considering the states to be fully observable in the next time step~\cite{littmanlearning}. Though majority of our (query) actions do not have value while computing $V_{qmdp}$, we still consider the QMDP value as it can serve as an upper bound.

\end{enumerate}

\begin{figure*}[t]
\centering
\hspace{2mm}
\epsfig{figure=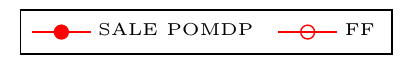, scale=1.6}\\
 \subfigure[FF Error]  {\epsfig{figure=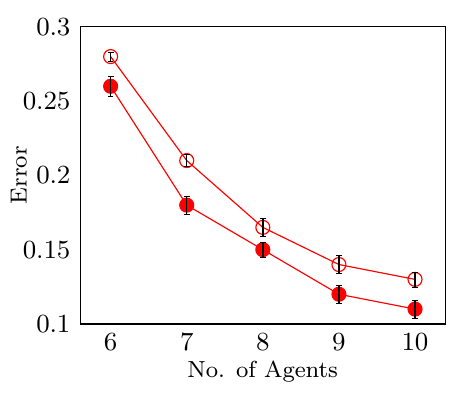, scale=1.4}}
  \subfigure[FF Value] {\epsfig{figure=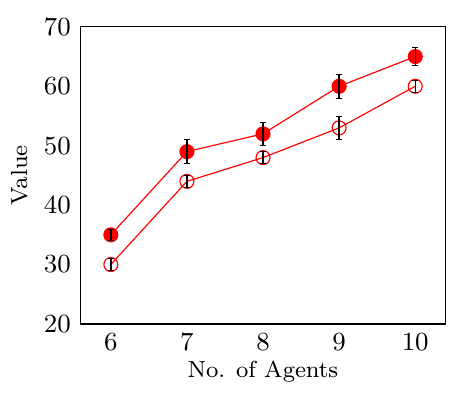, scale=1.4}}
\caption{Influence of FF algorithm} \label{fig:influence}
\end{figure*}
We conduct experiments in a simulated e-marketplace, where buyers need to
choose sellers as successful transaction partners. We measure the average
\emph{error} $\in [0,1]$ in terms of the percentage of
`unsuccessful transactions'
(buying from a bad seller or taking the $DNB$ action in the presence of a good seller)
and \emph{value}, i.e, the discounted ($0.95$) reward in the
process of choosing a seller. The buyer pays a cost of $1$ for querying
advisors about other advisors, $10$ for querying about a seller, gains $100$
for choosing a good seller or taking $DNB$ when no seller is of good
quality, loses $100$ for choosing a bad seller or taking $DNB$
when there is a good seller. The number of sellers is $20\%$ of
the whole population $W$ and number of advisors is $80\%$ among which $20\%$
are untrustworthy. All the results are values averaged over $500$ iterations from
the point of view of a single buyer. We consider single transaction settings, where the
buyer has no previous experience with the seller.

\begin{table}[b]
\centering
\scriptsize
\begin{tabularx}{\linewidth}{|X|X|X|X|X|X|X|X|X|X|}
\hline $W$    &6 &7 &8 &9   &10     &25     &50      &75      &100  \\
\hline
$|S|$     &$2^6$ &$2^7$ &$2^8$ &$2^9$ &$2^{10}$  &$2^{25}$  &$2^{50}$  &$2^{75}$ &$2^{100}$  \\
$|A|$   &$27$ &$38$ &$45$ &$59$ & $75$  &$486$ &$1971$ &$4456$  &$7941$ \\
\hline
\end{tabularx}
\caption{Size of the seller selection problem}
\label{table:size}
\end{table}

To analyze the scalability, we increase the number of agents $W$ in the e-marketplace from $6-100$ (size of the corresponding seller selection problem, is given in Table~\ref{table:size}) and measure the performance of the approaches in Fig.~\ref{fig:influence}-\ref{fig:overall}. As SALE POMDP does not scale effectively to more than $10$ agents (ran out of time while computing the policy), its performance is not shown for $W$$>$$10$ in the figures. 

\subsection{Influence of using Factored Frontier (FF)} Fig.~\ref{fig:influence}(a-b) show the influence of
using FF for the belief update.
We see that while the approximation introduced by FF leads to a reduction in value compared to using exact belief updates, the difference is quite small.

\ourparagraph{Analysis of Different Design Schemes for the Majority Voting MOPE Approach} Fig.~\ref{fig:majdesign} shows the analysis of the different design considerations, such as the performance of voting hierarchies, influence of SPA and APS for the Majority Voting MOPE approach. In Fig.~\ref{fig:majdesign}(a-b), we analyse the performance of the H1, H2 and H3 hierarchies while using the Majority Voting MOPE approach. We see that H3 hierarchy outperforms H1 and H2. Also, we see that for (most) cases where the SALE POMDP is able to provide an answer, it is performing slightly better than H3. This is expected since it does a full POMDP reasoning over the entire state space. However, for larger problems, the difference in performance becomes negligible and when including more advisors, H3 finds policies that lead to significantly smaller errors and higher payoffs. SingleExpert(5) achieves a constant performance as it always considers a group of $5$ agents to make decisions. The performance of all other approaches increase with the number of agents as there are more advisors to seek information about the sellers.

In Fig.~\ref{fig:majdesign}(c-d), we analyse the influence of the number of SPs per agent (SPA), using H3 Majority Voting (with default SPA=$4$). Fig.~\ref{fig:majdesign}(c-d) show that performance of H3 increases with SPA, i.e., H3S8 (SPA=$8$) shows the best performance and H3S2 (SPA=$2$) shows the least performance. This is because, on increasing SPA, the total number of SPs considered increase, resulting in more informed decision making. We see that H3 and H3S8 outperform SALE POMDP for $10$ agents, suggesting that the quality of Symbolic Perseus degrades for larger problems. Fig.~\ref{fig:majdesign}(e-f) show the influence of the number of agents per SP (APS) for the H3S8 technique.
H3S8A6 (APS=$6$), H3S8A7 (APS=$7$) and H3S8A8 (APS=$8$) outperform H3S8 (default APS=$5$) as increasing APS improves performance by reasoning over a larger state space.
Importantly, we see how this enables MOPE to accumulate a significantly higher value (H3S8A7 obtains a value of $72$ for $100$ agents) than the best SALE POMDP value ($65$ for $10$ agents).
We expect that the lower performance of H3S8A8 compared to H3S8A7 is caused by a relative degradation of the solution quality of the (larger) SPs. However, H3S8A6, H3S8A7 and H3S8A8 involve greater policy computation time than H3S8.

\begin{figure*}[!htbp]
\vspace{-6mm}
\centering
\hspace{3mm}
\epsfig{figure=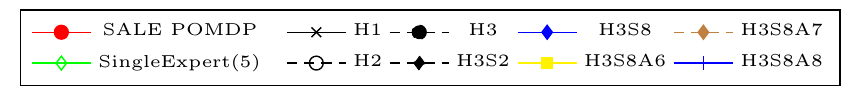, scale=1.6}\\ \vspace{-2mm}
 \subfigure[Voting Hierarchy Error]  {\epsfig{figure=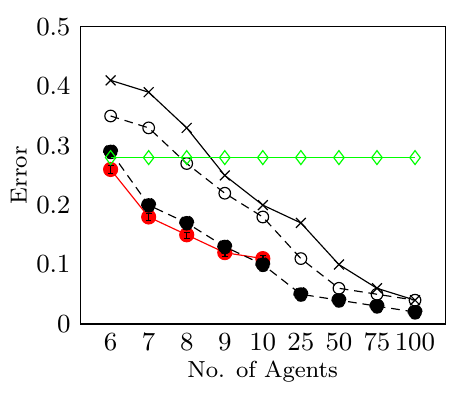, scale=1.4}} \vspace{-2mm}
\subfigure[Voting Hierarchy Value]   {\epsfig{figure=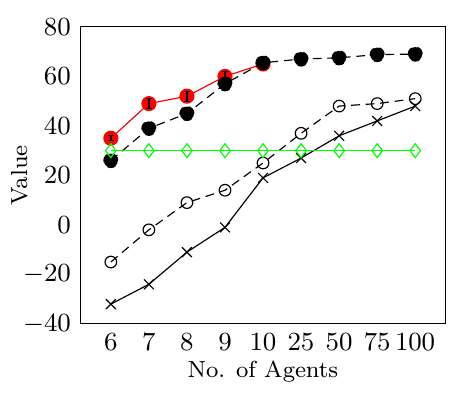, scale=1.4}} \vspace{-2mm}\\ \vspace{-2mm}
 \subfigure[SPA Error] {\epsfig{figure=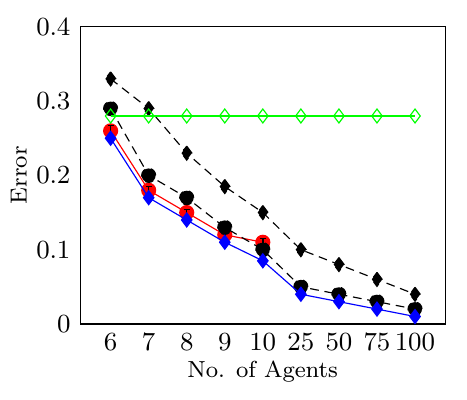, scale=1.4}} \vspace{-2mm}
\subfigure[SPA Value] {\epsfig{figure=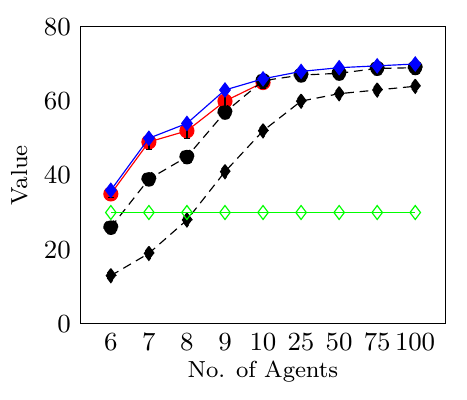, scale=1.4}}  \vspace{-2mm} \\ \vspace{-2mm}
\subfigure[APS Error] {\epsfig{figure=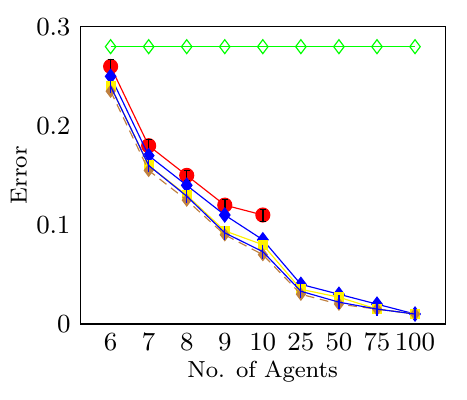, scale=1.4}} \vspace{-2mm}
 \subfigure[APS Value]  {\epsfig{figure=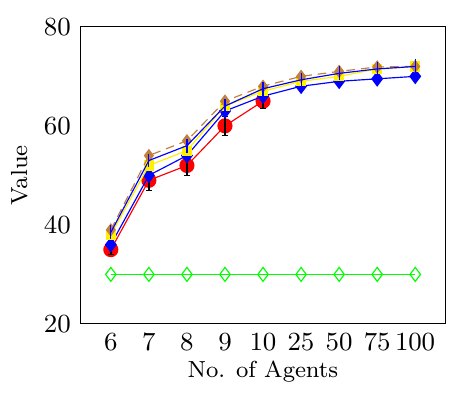, scale=1.4}}
\caption{Performance comparison of the Majority Voting design schemes: (a-b) voting hierarchies; (c-d) influence of SPA; (e-f) influence of APS} \label{fig:majdesign}
\vspace{-7mm}
\end{figure*}

\ourparagraph{Comparison with Max-Q, POMCP, $\text{\textbf{\emph{V}}}_{\text{\textbf{\emph{maxv}}}}$ and $\text{\textbf{\emph{V}}}_{\text{\textbf{\emph{qmdp}}}}$}
In Fig.~\ref{fig:overall}, we compare the performance of H3S8 along with Max-Q (SPA=$8$, APS=$5$ and using the FF algorithm for belief update). We have shown the error and value for the POMCP approach in Table~\ref{table:pomcp} separately, to retain the clarity in Fig.~\ref{fig:overall}(a-b). We see that H3S8 outperforms both Max-Q and POMCP. As the number of agents increases, performance of POMCP decreases, as it requires a larger number of simulations to sample the beliefs and histories about the agents. Also, POMCP does not scale well with the number of actions (which is large in these problems). We have not shown the POMCP results for $W>25$ due to the complexity of the simulations.
\begin{figure*}[!t]
\centering
\vspace{8mm}
\hspace{3mm}
\epsfig{figure=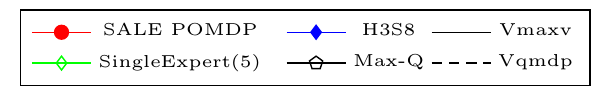, scale=1.6}\\
 \subfigure[Error]  {\epsfig{figure=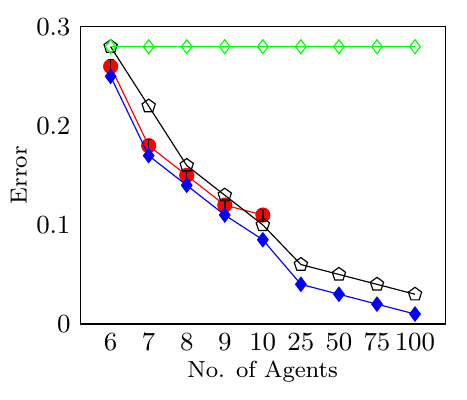, scale=1.4}}
 \subfigure[Value] {\epsfig{figure=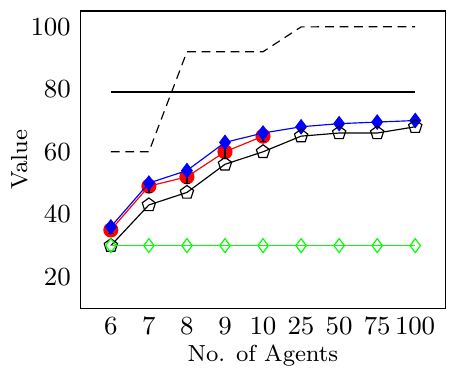, scale=1.4}}
\caption{Performance Comparison of the Majority Voting technique with Max-Q, $V_{maxv}$, $V_{qmdp}$} \label{fig:overall}
\end{figure*}

\begin{table}[!htbp]
\centering
\scriptsize
\begin{tabularx}{\linewidth}{|X|X|X|X|X|X|X|X|X|X|}
\hline $W$    &6 &7 &8 &9   &10     &25     &50      &75      &100  \\
\hline
error    &$0.60$ &$0.62$ &$0.70$ &$0.80$ &$0.81$  &$0.90$  &-  &- &-  \\
value   &$-46$ &$-54$ &$-68$ &$-90$ & $-92$  &$-110$ &- &-  &- \\
\hline
\end{tabularx}
\caption{Performance of POMCP}\label{table:pomcp}
\end{table}

Fig.~\ref{fig:overall}(b) also shows the $V_{maxv}$ and $V_{qmdp}$ values. Specifically, we consider the $V_{maxv}$ value, in order show the performance of the Majority Voting MOPE scheme (H3S8), under the most favourable conditions. We know that beginning with a most favourable state always results in best performance while executing a POMDP policy. In our case, the most favourable state represents the presence of good quality sellers and trustworthy advisors in the SPs, as they have a higher probability of resulting in successful transactions, thereby leading to greater value. As we can see, the $V_{maxv}$ value is greater than the value obtained by H3S8 for normal problems (in which sellers can also be of low quality and advisors can be untrustworthy). $V_{qmdp}$ is the upper bound value 
and looks like a piecewise function because of the same number of sellers in some of the problems. Thus, Fig.~\ref{fig:overall}(b) shows the lower bound, i.e., the value of SingleExpert(5), optimistic heuristic value $V_{maxv}$, and the upper bound $V_{qmdp}$ for a $5$-agent decomposition.

Fig.~\ref{fig:time} shows the policy computation time for each seller selection problem involving $6$ to $100$ agents. For POMCP, we measure the simulation time per episode. We see that the time taken by H3S8, SingleExpert(5), Max-Q is less than SALE POMDP and POMCP. Also, the constant time $22s$ for H3S8, SingleExpert(5) and Max-Q is due to using the same $5$-agent policy for all SPs.

\begin{figure}[t]
\centering
       \subfigure[Time Complexity]              {\epsfig{figure=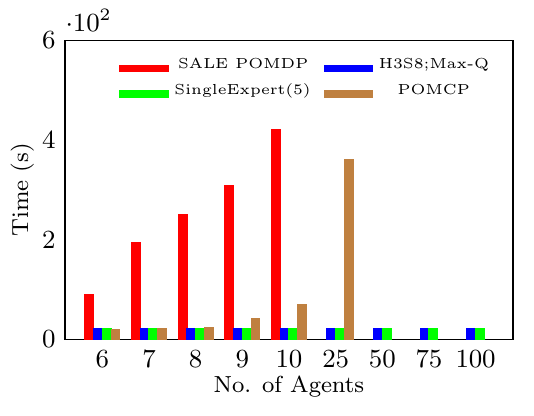, scale=1.3}}
\caption{Comparison of policy time} \label{fig:time}
\vspace{-2mm}
\end{figure}

\ourparagraph{Performance in WSN Domain} While the MOPE approach can improve the scalability of the SALE POMDP model in the e-marketplace domain (as shown in Fig.~\ref{fig:overall}), it can also be applied to improve the scalability of the POMDP models in other domains, which follow a similar trust propagation structure as the SALE POMDP. To verify this, we also apply the MOPE approach (H3S4 Majority Voting with APS=$3$, SPA=$4$) to improve the scalability of the SRP model (see \cite{irissappanesecure} for details) in the WSN domain and compare it with: 1) the original SRP model; and 2) SingleExpert(3) with $3$ agents.
We use the same simulation settings as used in \cite{irissappanesecure}.
Fig.~\ref{fig:wsn}(a-b) show that SRP performs better than H3S4 for $3-5$ neighbors. However, it cannot provide solutions for more than $5$ neighbors, while H3S4 can scale up to $40$ neighbors, generating a much higher value. Also, the policy computation time is $73s$ for H3S4 and $736s$ for the SRP model for $5$ neighbors.
\begin{figure*}[t]
\centering
\hspace{1mm}
\epsfig{figure=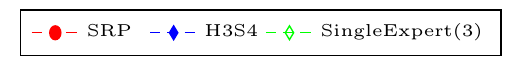, scale=1.6}\\
  \subfigure[WSN Error]              {\epsfig{figure=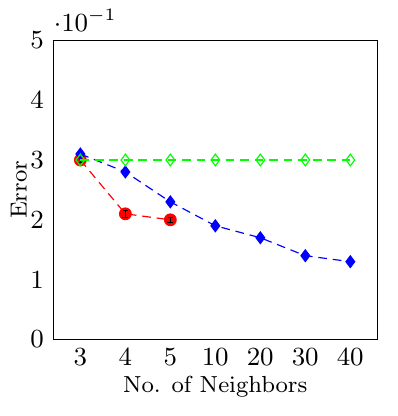, scale=1.4}}
         \subfigure[WSN Value]              {\epsfig{figure=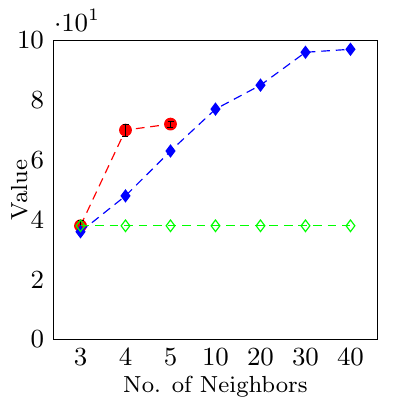, scale=1.4}}
\caption{Performance in wireless sensor networks} \label{fig:wsn}
\vspace{-2mm}
\end{figure*}

\section{Related Work}

 There is extensive literature on scalable solutions to solving POMDPs. Point-Based Value Iteration
(PBVI)~\cite{pineau2003point} computes a value function over a finite subset of the belief space. A point based algorithm explores the belief space, focusing on the reachable belief states, while maintaining a value function by applying the point-based backup operator. Bounded policy iteration (BPI)~\cite{poupart2003bounded} incrementally
constructs a finite state controller by alternating policy improvement and policy evaluation until a local optimum is reached by slowly increasing the number of nodes. Gradient ascent~\cite{aberdeen2002scaling} restricts its search to controllers of a bounded size. However, the above approaches scale only to thousands of states~\cite{poupart2004vdcbpi}.

In structured domains, further scaling can be achieved by exploiting compact representations~\cite{feng2001approximate,guestrin2001solving,veiga2014point,poupart2004vdcbpi}, such as decision trees~\cite{boutilier1996computing}, algebraic decision diagrams
(ADDs)~\cite{hansen2000dynamic}, or by
indirectly compressing the belief space into a small subspace by value-directed compression
(VDC)~\cite{poupart2005exploiting}, one of which is also applied in the regular SALE POMDP model.

While all the above are offline policy computation algorithms, recently an online POMDP planning algorithm called POMCP~\cite{silver2010monte} has successfully scaled up to very large problems. POMCP is based on Monte Carlo tree search, which tries to break the curse of dimensionality and history by sampling states from the current belief and histories with a black-box simulator. On the other hand, in our approach, we use offline policy computation~\cite{poupart2005exploiting}, to compute optimal policies for each sub-POMDP, while still achieving better scalability than POMCP (as shown in our experiments).

Some approaches use a similar concept of decomposing a
(PO)MDP into smaller sub-problems.
\citet{meuleau1998solving} assume that sub-problems are \emph{very weakly coupled}:
each sub-problem corresponds to an independent sub-task whose state/action
spaces do not directly influence the other tasks.
In contrast, MOPE divides a single large POMDP problem into SPs, which can contain overlapping
state variables/actions.
Similar to our work, \citet{williams2007scaling} consider a more general decomposition, but they
rely on domain specific heuristics, while we investigate several
general methods to aggregate the recommendations from all SPs.
\citet{yadav2015preventing} also propose an 
approach which decomposes a POMDP into SPs, 
but these are formed in a very different way: by sampling \emph{values} for sub-sets of hidden state factors. 
%
A major difference between all these works and ours, is that their 
sub-problems directly follow from the domain. In contrast, in our
approach, the number of sub-problems can be chosen to control the time vs.
quality trade-off.


Decomposition has also been a popular technique in multiagent planning
approaches~\cite{guestrin2001multiagent,becker2003transition,nair2003taming,goldman2008communication,witwicki2010influence,oliehoek2012influence,amato2015scalable,oliehoekfactored}.
However, in all these cases structure is exploited that is particular to the
multiagent setting by extending insights from factored (PO)MDP approaches and when
applied to single-agent problems such as a SALE POMDP these methods do not
offer any additional benefits.


MOPE can be interpreted as a type of ensemble method~\cite{dietterich2000ensemble}. In particular, there is a resemblance to random forests~\cite{breiman2001random}:
the way that they randomly select features is not unlike our random selection of state factors (seller and advisor variables).

\section{Conclusion and Future Work}
We propose the Mixture of POMDP Experts (MOPE) technique to address the scalability issues
in solving large seller selection (SALE) POMDP problems for e-marketplaces. MOPE works by dividing the large POMDP problem into computationally tractable smaller
sub-POMDPs and then aggregates the actions of the sub-POMDPs. Extensive evaluation shows that MOPE achieves a reasonable approximation to the SALE POMDP for small problems and can scale up to a hundred agents 
by effectively exploiting the presence of more advisors
to generate significantly higher buyer satisfaction. We also show that MOPE improves the scalability of a POMDP model in the sensor network domain.

We conduct experiments to select the best action hierarchy to be used in the MOPE approach. However, whether empirically determining good hierarchies for other problems (other than seller selection problems) will be possible is an interesting open question, which we would like to investigate as future work. We will also analyze more sophisticated ways (e.g., using community detection) of dividing the sub-POMDPs rather than random partitioning.

\section*{Acknowledgments}

This work is supported by the A*STAR SERC grant (1224104047) awarded to Dr. Jie Zhang, NWO Innovational Research Incentives Scheme Veni \#639.021.336 awarded to Dr. Frans A. Oliehoek and the Institute for Media Innovation at Nanyang Technological University.

\bibliographystyle{aaai}

\onecolumn
\setcounter{secnumdepth}{0}
\section*{Appendix A}

\section{Analysis of \pmaxq}

Here, we present an analysis of Parallel Max-Q. In this analysis, we assume
that the optimal values for the sub-POMDPs (SPs) can be computed exactly.
\paragraph{Parallel Max-Q.}

Here, we first summarize the essential characteristics of Parallel Max-Q.
Like all MOPE approaches, this method divides the large SALE POMDP into
$D=\left\{ \mathcal{M}_{1},\dots,\mathcal{M}_{K}\right\} $ sub-POMDPs
(SPs) randomly, and computes their optimal values $\left\{ V_{1}^{*},\dots,V_{K}^{*}\right\} $
before execution. What is specific for \pmaxq{} is that it treats
these SPs as SingleExperts that are maintained in parallel. The actions are
selected based on the maximizing Q-value in any SP.

\subparagraph{Beliefs:}

\pmaxq{} maintains, in parallel, a set $\mathcal{B}$ of local beliefs:
$\mathcal{B}=\left\{ b_{1},\dots,b_{K}\right\} $. These beliefs $b_{k}$
are defined over local states $s_{k}$.

\subparagraph{Belief update:}

The belief update operator $BU^{pmq}$ produces a new set of beliefs
$\mathcal{B}^{ao}$=$BU^{pmq}(\mathcal{B},a,o)$. In particular, each
sub-problem $\mathcal{M}_{k}$ is updated in parallel such that $\mathcal{B}^{ao}=\left\{ b_{1}^{ao},\dots,\dots,b_{K}^{ao}\right\} $, where,
\begin{eqnarray*}
 b_{k}^{ao}=\begin{cases}
BU(b_{k},a,o), & \mbox{if \ensuremath{a\in\mathcal{A}_{k}} and \ensuremath{o\in\mathcal{O}_{k}}}\\
b_{k}, & \text{otherwise}
\end{cases}
\end{eqnarray*}
That is, each local belief is updated (using the Bayesian belief update
operator $BU$) if the taken action and received observation exist
in that SP.

\subparagraph{Action Selection:}

At every stage, the \pmaxq{} selects the `winning' SP
\begin{equation}
\bar{k}=\arg\max_{k\in\left\{ 1,\dots,K\right\} }V_{k}^{*}(b_{k})\label{eq:pmaxq:selected_k}
\end{equation}
and subsequently executes its maximizing action:
\begin{equation}
\bar{a}=\arg\max_{a\in\mathcal{A}_{\bar{k}}}Q_{\bar{k}}^{*}(b_{\bar{k}},a)\label{eq:pmaxq:selected_a}
\end{equation}
Here, the Q-value is the standard Q-value for a POMDP:
\[
Q_{\bar{k}}^{*}(b_{\bar{k}},a)=R_{\bar{k}}(b_{\bar{k}},a)+\gamma\sum_{o\in\mathcal{O}_{\bar{k}}}\Pr(o|b,a)V_{\bar{k}}^{*}(b_{\bar{k}}^{ao})
\]
(The Q-value functions can be pre-computed along with the $\left\{ V_{1}^{*},\dots,V_{K}^{*}\right\} $
or the back-projection can be computed online). We will also denote
this action selection operation as $Act^{pmq}()$ such that $\bar{a}\triangleq Act^{pmq}(\mathcal{B})$.

\paragraph{\pmaxq{} Value Function.}

Here, we analyze the value realized by \pmaxq{}by giving
a formulation of its value function. We first identify the dependencies
of this value function:
\begin{itemize}
\item It clearly depends on the chosen decomposition $D$.
\item Due to the action selection mechanism, this value is dependent on
$\left\{ V_{1}^{*},\dots,V_{K}^{*}\right\} $, the value functions
of the SPs. Since these are implied by $D$, we will leave this dependence
implicit.
\item Clearly, via the same mechanism, the value also depends on $\mathcal{B}$.
\item However, since the belief updating process of \pmaxq{}is different
from the Bayes' rule, $\mathcal{B}$ might not correspond to the \emph{true
}posterior distribution over states. Clearly, the \emph{true }value
will depend on such a true distribution. Therefore, we make the true
state distribution explicit and denote it with $\beta$.
\end{itemize}
We address the infinite-horizon setting, but for sake of argument,
let us suppose there is only $\tau=1$ step-to-go. In such a case,
the expected reward is given by,
\begin{eqnarray}
V_{D}^{pmq}(\beta,\mathcal{B}) & = & \sum_{s}\beta(s)R(s,Act^{pmq}(\mathcal{B}))\nonumber \\
\text{\{we only get reward from selected SP \ensuremath{\bar{k}}\}} & = & \sum_{s}\beta(s)R_{\bar{k}}(s_{\bar{k}},\bar{a})\nonumber \\
 & = & \sum_{s_{\bar{k}}}\beta(s_{\bar{k}})R_{\bar{k}}(s_{\bar{k}},\bar{a})\label{eq:R(Beta,B)}
\end{eqnarray}
We will denote this quantity with $R(\beta,\mathcal{B})$. Now, let
us consider the general case where there are $\tau>1$ steps-to-go.
The expected immediate reward will be the same as described before,
but we need to add an expected future reward term:
\begin{equation}\label{eqn:vpmq1}
V_{D}^{pmq}(\beta,\mathcal{B})=R(\beta,\mathcal{B})+\gamma\mathbf{E}\left[V_{D}^{pmq}(\beta',\mathcal{B}')\right]
\end{equation}
Here, the updated beliefs $\beta'$ and $\mathcal{B}'$ depend on the
taken action $\bar{a}=Act^{pmq}(\mathcal{B})$ and received observation
$o$: $\beta'=BU(\beta,\bar{a},o)$ and $\mathcal{B}'=BU^{pmq}(\mathcal{B},\bar{a},o)$.
This means that we can give an explicit formulation of the value function
as follows:
\begin{lem}
\label{lem:pmaxq-value}The value function for \pmaxq{}is given
by
\[
V_{D}^{pmq}(\beta,\mathcal{B})=R(\beta,\mathcal{B})+\gamma\sum_{o}\Pr(o|\beta,\bar{a})V_{D}^{pmq}(BU(\beta,\bar{a},o),BU^{pmq}(\mathcal{B},\bar{a},o))
\]
\end{lem}

\begin{proof}
We substitute for the updated beliefs $\beta', \mathcal{B}'$ in Eqn.~\ref{eqn:vpmq1} and make the
observation probability explicit.
\end{proof}
\noindent
As the presence of both $\beta,\mathcal{B}$ complicates the analysis,
we introduce criteria under which they coincide:
\begin{defn}\label{def:equal-beliefs}
We call $\beta,\mathcal{B}$ \emph{equivalent} if
\[
\forall_{s}\quad\beta(s)=\prod_{i=1}^{K}b_{k}(s_{k})
\]
\end{defn}
\begin{lem}
\label{lem:equal-beliefs-under-conditions} When the following conditions hold:
\begin{enumerate}
\item The decomposition $D=\left\{ \mathcal{M}_{1},\dots,\mathcal{M}_{K}\right\} $
is non-overlapping, meaning that no two SPs $\mathcal{M}_{i},\mathcal{M}_{j}$
contain the same (seller- or advisor-) state factors;
\item The true initial state distribution $\beta^{0}$ is factored along
the decomposition: $\beta^{0}(s)=\beta_{1}^{0}(s_{1}) \times \beta_{2}^{0}(s_{2}) \cdot \dots \times \beta_{K}^{0}(s_{K})$;%
\footnote{\textit{\emph{There could be state factors not covered by the decomposition,
but these can be ignored for purposes of this proof.}}%
}
\item Only actions and observations that are contained in one of the SPs
are executed/received;
\end{enumerate}

then \pmaxq{} maintains the correct Bayesian posterior. That is,
for all possible histories $h$ of actions and observations, the induced
beliefs by $BU$, $\beta^{h}$, are factored and therefore equivalent
to those induced by $BU^{pmq}$: $\beta^{h}\equiv\mathcal{B}^{h}$.

\end{lem}
\begin{proof}
Via induction. For $t=0$, \emph{$\mathcal{B}^{0}$} is initialized
as $\left\{ \beta_{1}^{0},\dots,\beta_{K}^{0}\right\} $ and via assumption
$2$, $\beta^{0}$ is factored. $\mathcal{B}^{0}$ and $\beta^{0}$ are
therefore equivalent. Now, we prove that given $\beta^{t}$ is factored, $\beta^{t},\mathcal{B}^{t}$
are equivalent, so are $\beta^{t+1},\mathcal{B}^{t+1}$. Assume an
arbitrary $a,o$ satisfying condition 3. Then,
\begin{eqnarray}
\forall_{s'}\qquad\beta^{t+1}(s') & = & BU(\beta^{t},a,o) \nonumber\\
 & = & \frac{O(o|a,s')\sum_{s}\beta^{t}(s)T(s'|s,a)}{\Pr(o|\beta^{t},a)} \label{eqn:bu}
\end{eqnarray}
Let $k$ denote the SP to which $a$ belongs. Since per definition
only actions that interact with the state factors in $k$ are contained
in $k$, and there is no overlap between state factors in different
SPs, we have that $O(o|a,s')=O_{k}(o|a,s'_{k})$. Similarly, we have
that $T(s'|s,a)=\prod_{l=1}^{K}T(s_{l}'|s_{l},a)$. Therefore, we can
rewrite the numerator of the belief update in  Eqn.~\ref{eqn:bu} as follows,
\begin{eqnarray*}
 &  & O_{k}(o|a,s'_{k})\sum_{s}\left[\beta_{1}^{t}(s_{1})\cdot\dots\cdot\beta_{K}^{t}(s_{K})\right]\prod_{l=1}^{K}T(s_{l}'|s_{l},a)\\
 & = & O_{k}(o|a,s'_{k})\sum_{s_{1}}\cdots\sum_{s_{K}}\prod_{l=1}^{K}\beta_{l}^{t}(s_{l})T(s_{l}'|s_{l},a)\\
 & = & O_{k}(o|a,s'_{k})\left(\sum_{s_{1}}\beta_{1}^{t}(s_{1})T(s_{1}'|s_{1},a)\right)\cdots\left(\sum_{s_{K}}\beta_{K}^{t}(s_{K})T(s_{K}'|s_{l},a)\right)\\
 & = & \left(O_{k}(o|a,s'_{k})\sum_{s_{k}}\beta_{k}^{t}(s_{k})T(s_{k}'|s_{k},a)\right)\prod_{l\neq k}\left(\sum_{s_{l}}\beta_{l}^{t}(s_{l})T(s_{l}'|s_{l},a)\right)
\end{eqnarray*}
Similarly, the denominator can be written as:
\begin{eqnarray*}
 &  & \sum_{s_{1}'}\cdots\sum_{s_{K}'}\left(O_{k}(o|a,s'_{k})\sum_{s_{k}}\beta_{k}^{t}(s_{k})T(s_{k}'|s_{k},a)\right)\prod_{l\neq k}\left(\sum_{s_{l}}\beta_{l}^{t}(s_{l})T(s_{l}'|s_{l},a)\right)\\
 & = & \sum_{s_{k}'}\left(O_{k}(o|a,s'_{k})\sum_{s_{k}}\beta_{k}^{t}(s_{k})T(s_{k}'|s_{k},a)\right)\prod_{l\neq k}\left(\sum_{s_{l}'}\sum_{s_{l}}\beta_{l}^{t}(s_{l})T(s_{l}'|s_{l},a)\right)\\
 & = & \sum_{s_{k}'}\left(O_{k}(o|a,s'_{k})\sum_{s_{k}}\beta_{k}^{t}(s_{k})T(s_{k}'|s_{k},a)\right)\prod_{l\neq k}1\\
 & = & \sum_{s_{k}'}\left(O_{k}(o|a,s'_{k})\sum_{s_{k}}\beta_{k}^{t}(s_{k})T(s_{k}'|s_{k},a)\right)
\end{eqnarray*}
Such that the entire $BU$ operation can be written as,
\begin{eqnarray*}
\beta^{t+1}(s') & = & \frac{\left(O_{k}(o|a,s'_{k})\sum_{s_{k}}\beta_{k}^{t}(s_{k})T(s_{k}'|s_{k},a)\right)\prod_{l\neq k}\left(\sum_{s_{l}}\beta_{l}^{t}(s_{l})T(s_{l}'|s_{l},a)\right)}{\sum_{s_{k}'}\left(O_{k}(o|a,s'_{k})\sum_{s_{k}}\beta_{k}^{t}(s_{k})T(s_{k}'|s_{k},a)\right)}\\
 & = & \frac{O_{k}(o|a,s'_{k})\sum_{s_{k}}\beta_{k}^{t}(s_{k})T(s_{k}'|s_{k},a)}{\sum_{s_{k}'}\left(O_{k}(o|a,s'_{k})\sum_{s_{k}}\beta_{k}^{t}(s_{k})T(s_{k}'|s_{k},a)\right)}\prod_{l\neq k}\left(\sum_{s_{l}}\beta_{l}^{t}(s_{l})T(s_{l}'|s_{l},a)\right)
\end{eqnarray*}
which shows that $\beta^{t+1}$ is factored. Moreover,
\[
\beta_{k}^{t+1}(s_{k}')=\frac{O_{k}(o|a,s'_{k})\sum_{s_{k}}\beta_{k}^{t}(s_{k})T(s_{k}'|s_{k},a)}{\sum_{s_{k}'}\left(O_{k}(o|a,s'_{k})\sum_{s_{k}}\beta_{k}^{t}(s_{k})T(s_{k}'|s_{k},a)\right)}\triangleq BU(b_{k},a,o),
\]
and---since $a$ is an action of SP $k$---this is exactly the component
that $BU^{pmq}$ maintains. Due to the fact that the transitions in
the SALE POMDP are static (states do not change)
\[
\forall_{l\neq k}\quad\beta_{l}^{t+1}(s_{l}')=\sum_{s_{l}}\beta_{l}^{t}(s_{l})T(s_{l}'|s_{l},a)=\beta_{l}^{t}(s_{l})
\overset{I.H.}{=} b_{l}^{t}(s_{l})
\]
Here, the last equality follows from the induction hypothesis. Since
$a\not\in\mathcal{A}_{l}$, $b_{l}^{t}(s_{l})$ is exactly the component
that $BU^{pmq}$ maintains. This shows that $\forall_{k}\quad\beta_{k}^{t+1}=b_{k}^{t+1},$
thus proving the lemma.
\end{proof}
This means that under stated assumptions, we can now simplify the
value of Parallel Max-Q:
\begin{cor}
\label{cor:pmq-value-simpler}Under the conditions of Lemma~\ref{lem:equal-beliefs-under-conditions},
the value function for \pmaxq{}can be simplified to:

\[
V_{D}^{pmq}(\mathcal{B})=R_{\bar{k}}(b_{\bar{k}},\bar{a})+\gamma\sum_{o\in\mathcal{O}_{\bar{k}}}O_{\bar{k}}(o|b_{\bar{k}},\bar{a})V_{D}^{pmq}(BU^{pmq}(\mathcal{B},\bar{a},o)),
\]
with
\[
R_{\bar{k}}(b_{\bar{k}},\bar{a})\triangleq\sum_{s_{\bar{k}}}b_{\bar{k}}(s_{\bar{k}})R_{\bar{k}}(s_{\bar{k}},\bar{a}),
\]
and
\[
O_{\bar{k}}(o|b_{\bar{k}},\bar{a})=\sum_{s_{\bar{k}}'}O_{\bar{k}}(o|\bar{a},s_{\bar{k}}')\sum_{s_{\bar{k}}}b_{\bar{k}}(s_{\bar{k}})T_{\bar{k}}(s_{\bar{k}}'|s_{\bar{k}},\bar{a})
\]
\end{cor}
\begin{proof}
The proof of Lemma~\ref{lem:equal-beliefs-under-conditions} showed
that, under stated conditions, $\mathcal{B}$ represents the true
posterior, such that $\beta(s_{k})=b_{k}(s_{k})$. Therefore, the definition
of $R(\beta,\mathcal{B})$ given by Eqn.~\ref{eq:R(Beta,B)} simplifies
to the above definition of $R_{\bar{k}}(b_{\bar{k}},\bar{a})$. Similarly,
the observation probability only depends on $s_{\bar{k}}$ (due to
the structure of the SALE POMDP) and thus $b_{\bar{k}}$.
\end{proof}

\paragraph{Lower Bound on the Value of Parallel Max-Q.}

Here, we analyze the performance of Parallel Max-Q, giving a lower bound
on its performance. In particular, we show that the $V_{D}^{pmq}$
value realized by \pmaxq{}for a given decomposition $D=\left\{ \mathcal{M}_{1},\dots,\mathcal{M}_{K}\right\} $
is at least as much as the optimal \se{}value $V_{k}^{*}$ that
would be realized by picking \emph{any} sub-POMDP $\mathcal{M}_{k}$.

In more detail, let us define, given $D$, a method `\bse{}' that
selects the best \se{}and sticks with it, then we can prove that
the value of \bse{} is a lower bound to that of Parallel Max-Q.
\begin{thm}
If the decomposition $D=\left\{ \mathcal{M}_{1},\dots,\mathcal{M}_{K}\right\} $
is non-overlapping, and the true initial state distribution $\beta^{0}$
is factored along the decomposition, then the value realized by \pmaxq{}is at least as much as the value of the \bse{}:
\[
V_{D}^{pmq}(\mathcal{B}^0)
\geq
\max_{k\in\left\{ 1,\dots,K\right\} }V_{k}^{*}(\beta_{k}^0)
\]
\end{thm}
\begin{proof}
    We will actually prove the stronger statement that says that the inequality holds for
    any two equivalent $\mathcal{B}, \{\beta_1,\dots,\beta_K \}$:
\[
V_{D}^{pmq}(\mathcal{B})\geq\max_{k\in\left\{ 1,\dots,K\right\} }V_{k}^{*}(\beta_{k})
\]
Let us first consider the right hand side. The maximizing value is
produced by a maximizing action ($\bar{a}$):
\begin{equation}\label{eq:maxk}
\max_{k\in\left\{ 1,\dots,K\right\} }V_{k}^{*}(\beta_{k})=\max_{a\in\mathcal{A}_{\bar{k}}}Q_{\bar{k}}^{*}(\beta_{\bar{k}},a)=Q_{\bar{k}}^{*}(\beta_{\bar{k}},\bar{a}),
\end{equation}
where,
\begin{equation}
\bar{k}=\arg\max_{k\in\left\{ 1,\dots,K\right\} }V_{k}^{*}(\beta_{k}),\label{eq:bar(k)}
\end{equation}
\begin{equation}
\bar{a}=\arg\max_{a\in\mathcal{A}_{\bar{k}}}Q_{\bar{k}}^{*}(\beta_{\bar{k}},a)\label{eq:bar(a)}
\end{equation}
and $\bar{k}$ denotes the maximizing SP. The (regular POMDP) optimal Q-value for
SP $\bar{k}$ is defined as:

\[
Q_{\bar{k}}^{*}(\beta_{\bar{k}},\bar{a})=R_{\bar{k}}(\beta_{\bar{k}},\bar{a})+\gamma\sum_{o\in\mathcal{O}_{\bar{k}}}\Pr(o|\beta_{\bar{k}},\bar{a})V_{\bar{k}}^{*}(BU(\beta_{\bar{k}},\bar{a},o))
\]
Since, neither \bse{} nor \pmaxq{} will select actions outside of
the decomposition $D$, Lemma~\ref{lem:equal-beliefs-under-conditions}
asserts that the components $b_{\bar{k}}$ maintained by \pmaxq{}
are identical to $\beta_{\bar{k}}$.
Thus,
\[
Q_{\bar{k}}^{*}(\beta_{\bar{k}},\bar{a})=Q_{\bar{k}}^{*}(b_{\bar{k}},\bar{a})=R_{\bar{k}}(b_{\bar{k}},a)+\gamma\sum_{o\in\mathcal{O}_{\bar{k}}}\Pr(o|b_{\bar{k}},\bar{a})V_{\bar{k}}^{*}(BU(b_{\bar{k}},a,o))
\]
and by Eqn.~\ref{eq:maxk},
\[
\max_{k\in\left\{ 1,\dots,K\right\} }V_{k}^{*}(\beta_{k}) = \max_{k\in\left\{ 1,\dots,K\right\} }V_{k}^{*}(b_{k})
\]
Now, we observe that \pmaxq{}also first selects the SP with the
highest current value (i.e., $\bar{k})$ and then executes the maximizing
action (i.e, $\bar{a}$) specified by SP $\bar{k}$. As such the only
difference is that the future values for \pmaxq{}are different,
since it might switch to another sub-problem at the next stage (cf.
Lemma \ref{cor:pmq-value-simpler}):
\[
V_{D}^{pmq}(\mathcal{B})=R_{\bar{k}}(b_{\bar{k}},\bar{a})+\gamma\sum_{o\in\mathcal{O}_{\bar{k}}}O_{\bar{k}}(o|b_{\bar{k}},\bar{a})V_{D}^{pmq}(BU^{pmq}(\mathcal{B},\bar{a},o))
\]
Clearly, this suggests a proof via induction. Let $V_{D,\tau}^{pmq}$
and $V_{k,\tau}^{*}$ denote the values of performing $\tau$ steps
of \pmaxq{}and \se$(k)$ respectively.

\subparagraph{Base Case.}

For $\tau=1$ step-to-go, the above analysis shows that both \pmaxq{}and \bse{} will realize $R_{\bar{k}}(b_{\bar{k}},\bar{a})$.

\subparagraph{Induction Step.}

The Induction Hypothesis is that
\[
\forall_{\mathcal{B}}\quad V_{D,\tau}^{pmq}(\mathcal{B})\geq\max_{k\in\left\{ 1,\dots,K\right\} }V_{k,\tau}^{*}(b_{k}),
\]
which means we have to prove that
\[
\forall_{\mathcal{B}}\quad V_{D,\tau+1}^{pmq}(\mathcal{B})\geq\max_{k\in\left\{ 1,\dots,K\right\} }V_{k,\tau+1}^{*}(b_{k}).
\]

Proof: Assume an arbitrary $\mathcal{B}$, again $\bar{k}$ denotes
the maximizing SP in \pmaxq{}and it specifies action $\bar{a}$.
We then have,
\begin{eqnarray*}
V_{D,\tau+1}^{pmq}(\mathcal{B}) & = & R_{\bar{k}}(b_{\bar{k}},\bar{a})+\gamma\sum_{o\in\mathcal{O}_{\bar{k}}}O_{\bar{k}}(o|b_{\bar{k}},\bar{a})V_{D}^{pmq}(BU^{pmq}(\mathcal{B},\bar{a},o))\\
\text{\{I.H.\}} & \geq & R_{\bar{k}}(b_{\bar{k}},\bar{a})+\gamma\sum_{o\in\mathcal{O}_{\bar{k}}}O_{\bar{k}}(o|b_{\bar{k}},\bar{a})\max_{k\in\left\{ 1,\dots,K\right\} }V_{k,\tau}^{*}(b_{k}^{\bar{a}o})\\
 & \geq & R_{\bar{k}}(b_{\bar{k}},\bar{a})+\gamma\sum_{o\in\mathcal{O}_{\bar{k}}}O_{\bar{k}}(o|b_{\bar{k}},\bar{a})V_{\bar{k},\tau}^{*}(b_{k}^{\bar{a}o})\\
 & = & Q_{\bar{k},\tau+1}^{*}(b_{\bar{k}},\bar{a})\\
\text{\{per def. \eqref{eq:bar(a)}, \ensuremath{\bar{a}\,\,}is maximizing in \ensuremath{\bar{k}\}}} & = & V_{\bar{k},\tau+1}^{*}(b_{\bar{k}})\\
\text{\{per def. \eqref{eq:bar(k)} \ensuremath{\bar{k\,\,}}is the maximizing SP\}} & = & \max_{k\in\left\{ 1,\dots,K\right\} }V_{k,\tau+1}^{*}(b_{k})
\end{eqnarray*}
which concludes the proof.

\end{proof}

\end{document}